\documentclass{article}

\usepackage{microtype}
\usepackage{graphicx}
\usepackage{subfigure}
\usepackage{booktabs} 
\usepackage[colorinlistoftodos]{todonotes}
\usepackage[inline]{enumitem}
\usepackage{color}
\usepackage[capitalise]{cleveref}
\usepackage{bbm}
\usepackage{amssymb}
\usepackage{amsthm}
\usepackage{amsmath}
\usepackage{blindtext,titlefoot}
\usepackage{times}

\newtheorem{assumption}{Assumption}
\crefname{assumption}{Assumption}{Assumptions}
\newtheorem{theorem}{Theorem}
\newtheorem{lemma}{Lemma}
\newtheorem{lemma_ap}{Lemma}[section]

\newtheorem{corollary}{Corollary}

\Crefname{ALC@unique}{Line}{Lines}

\providecommand{\abs[1]}{\left|#1\right|}
\providecommand{\norm[1]}{\abs{\abs{#1}}}%
\providecommand{\hw}{\widehat{W}}
\providecommand{\res}{\Omega}
\providecommand{\red}[1]{{#1}}
\providecommand{\prt}[1]{\left(#1 \right)}

\def\R{\mathbb{R}}
\def\bY{\mathbf{Y}}

\newcommand{\e}{{\bf e}}

\usepackage{hyperref}

\usepackage[accepted]{style_files/icml2019}

\icmltitlerunning{Learning from Pairwise Comparisons}

\begin{document}

\twocolumn[
\icmltitle{Graph Resistance and Learning from Pairwise Comparisons}



\icmlsetsymbol{equal}{*}

\begin{icmlauthorlist}
\icmlauthor{Julien M. Hendrickx$^{*,\dag}$}{}
\icmlauthor{Alex Olshevsky$^\dag$}{}
\icmlauthor{Venkatesh Saligrama$^\dag$}{}
\end{icmlauthorlist}
\vskip 0.3in
]
\icmlcorrespondingauthor{Alex Olshevsky}{alexols@bu.edu}





\printAffiliationsAndNotice{\icmlEqualContribution} 

\begin{abstract} 

We consider the problem of learning the qualities of a collection of items by performing noisy comparisons among them. Following the standard paradigm, we assume there is a fixed ``comparison graph'' and every neighboring pair of items in this graph is compared $k$ times according to the Bradley-Terry-Luce model (where the probability than an item wins a comparison is proportional the item quality). We are interested in how the relative error in quality estimation scales with the comparison graph in the regime where $k$ is large. We prove that, after a known transition period, the relevant graph-theoretic quantity is the square root of the resistance of the comparison graph. Specifically, we provide an algorithm that is minimax optimal. The algorithm has a relative error decay that scales with the square root of the graph resistance, and provide a matching lower bound (up to log factors). The performance guarantee of our algorithm, both in terms of the graph and the skewness of the item quality distribution,  outperforms earlier results.

\end{abstract}

\unmarkedfntext{$ $\\$^*$ Department of Mathematical Engineering, ICTEAM, 
UCLouvain, Belgium\\
$^\dag$ Department of Electrical and Computer Engineering,
Boston University, USA
}

\section{Introduction}

This paper considers quality estimation from pairwise comparisons,  which is a common method of preference elicitation from users. For example, the preference of a customer for one product over another can be thought of as the outcome of a comparison. Because customers are idiosyncratic,  such outcomes will be noisy functions of the quality of the underlying items.  A similar problem arises in crowdsourcing systems, which must strive for accurate inference even in the presence of unreliable or error-prone participants.   Because crowdsourced tasks pay relatively little, errors are common; even among workers making a genuine effort, inherent ambiguity in the task might lead to some randomness in the outcome. These considerations make the underlying estimation algorithm an important part of any crowdsourcing scheme.

Our goal is accurate inference of true item quality from a collection of outcomes of noisy comparisons. We will use one of the simplest parametric models for the outcome of comparisons, the  Bradley-Terry-Luce (BTL) model, which associates a real-valued quality measure to each item and posits that customers select an item with a probability that is proportional to its quality. {\em Given a ``comparison graph'' which captures which pairs of items are to be compared, our goal is to understand how accuracy scales in terms of this graph when participants make choices according to the BTL model. }

We focus on the regime where we perform many comparisons of each pair of items in the graph. In this regime, we are able to give a satisfactory answer to the underlying question. Informally, we prove that, up to various constants and logarithms, the relative estimation error will scale with the square root of measures of resistance in the underlying graph. Specifically, we propose an algorithm whose performance scales with graph resistance, as well as a matching lower bound. The difference between our upper and lower bounds depends only on the log of the confidence level and on the skewness of the item qualities. Additionally, we note that our performance guarantees scale better in terms of item skewness as compared to previous work. 

\subsection{Formal problem statement}

We are given an undirected ``comparison graph'' $G(V,E)$, where each node $i$ has a positive weight $w_i$. If $(i,j)\in E$, then we perform $k$ comparisons between $i$ and $j$. The outcomes of these comparisons are i.i.d. Bernoulli and the probability that $i$ wins a given comparison according to the BTL model is   
\begin{equation} \label{eq:btl}
p_{ij} = \frac{w_i}{w_i+w_j} 
\end{equation}
The goal is to recover the weights  $w_i$ from the outcomes of these comparisons. Because multiplying all $w_i$ by the same constant does not affect the distribution of outcomes, we will recover a scaled version of the weight vector $w$.

Thus our goal will thus be come up with a vector of estimated weights $\widehat W$ close, in a scale-invariant sense, to the true but unknown vector\footnote{We follow the usual convention of denoting random variables by capital letters, which is why $\widehat{W}$ is capitalized while $w$ is not.} $w$. A natural error measure turns out to be the absolute value of the sine of the angle defined by $w$ and $\hw$, which  can also be expressed as (see Lemma \ref{lem:sinus} in the Supplementary Information)
\begin{equation}\label{eq:error_quad_alpha}
\abs{\sin(\hw,w)} = \inf_{\alpha \in \R}  \frac{|| \widehat W - \alpha w||_2}{|| \alpha w||_2}.
\end{equation}  In other words, $|\sin(\hw,w)|$ is the relative error to the closest normalization of the true quality vector $w$. We will also discuss the connection between this error measure and others later on in the paper. 

Following earlier literature, we assume that
\begin{eqnarray*} 
\max_{i,j\in V} \frac{w_i}{w_j} & \leq&  b 
\end{eqnarray*} 
for some constant $b$. The number $b$ can be thought of as a measure of the skewness of the underlying item quality. Our goal is to understand how the error between $\widehat{W}$ and $w$ scales as a function of the comparison graph $G$.

\subsection{Literature Review} 

The dominant approach to recommendation systems relies on inferring item quality from raw scores provided by users (see \cite{jannach2016recommender}).  However, such scores might be poorly calibrated and inconsistent; alternative approaches that offer simpler choices  might perform better.

Our starting point is the Bradley-Terry-Luce (BTL) model of Eq. (\ref{eq:btl}), dating back to  \cite{bradley1952rank, luce2012individual}, which models how individuals make noisy choices between items. A number of other models in the literature have also been used as the basis of inference, we mention the Mallows model introduced in \cite{mallows1957non} and the PL and Thurstone models (see description in \cite{hajek2014minimax}). However, we focus here solely on the BTL model. 

Our work is most closely related to the papers  \cite{negahban2012iterative} and \cite{negahban2016rank}. These works proposed an eigenvector calculation which, provided the number of comparisons is sufficiently large, successfully recovers the true weights $w$ from the outcomes of noisy comparisons. The main result of \cite{negahban2016rank} stated that, given a comparison graph, if the number of comparisons per edge satisfied a certain lower bound, then it is possible to construct an estimate $\widehat{W}$ satisfying 
\begin{equation}\label{eq:bound_negahban} \frac{||\widehat W - w||_2}{||w||_2} \leq 
 O \left( \frac{b^{5/2} d_{\rm max}}{d_{\rm min} (1-\lambda)} \sqrt{\frac{\log n}{k d_{\rm max}}} \right) 
\end{equation} with high probability, where $d_{\rm min}, d_{\rm max}$ are, respectively, the smallest and largest degrees in the comparison graph,  $1-\lambda$ is the spectral gap of a certain normalized Laplacian of the comparison graph, and both $w,\hw$ are normalized so that their entries sum to 1. It can be proved (see Lemma \ref{lem:compare_criteria}) that the relative error on the left-hand side of Eq. (\ref{eq:bound_negahban}) is within a $\sqrt{b}$ factor of the measure $|\sin(\widehat{W},w)|$ provided that $\max_{i,j} \widehat{W}_i/\widehat{W}_j \leq b$, so asymptotically these two measures differ only by factor depending on the skewness $b$.

The problem of recovering $w$ was further studied in \cite{rajkumar2014statistical}, where the comparison graph was taken to be a complete graph but with comparisons on edges  made at non-uniform rates. The sample complexity of recovering the true weights was provided as a function of the smallest sampling rate over pairs of items. 

A somewhat more general setting was considered in \cite{shah2016estimation}, which considered a wider class of noisy comparison models which include the BTL model as a special case. Upper and lower bounds on the minimax optimal rates in estimation, depending on the eigenvalues of a corresponding Laplacian, were obtained for absolute error in several different metrics; in one of these metric, the Laplacian semi-metric, the upper and lower bounds were tight up to constant factors. Similarly to \cite{shah2016estimation}, our goal is to understand the dependence on the underlying graph, albeit in the simpler setting of the BTL model. 

Our approach to the problem very closely parallels the approach of \cite{jiang2011statistical}, where a collection of potentially inconsistent rankings is optimally reconciled by solving an optimization problem over the comparison graph. However, whereas \cite{jiang2011statistical} solves a linear programming problem, we will use a linear least squares approach, after a certain logarithmic change of variable.  

We now move on to discuss work more distantly related to the present paper. We mention that the problem we study here is related, but not identical, to the  so-called noisy sorting problem, introduced in \cite{braverman2009sorting}, where better items win with probability at least $1/2 + \delta$ for some positive $\delta$. This assumption does not hold for the BTL model with arbitrary weights. Noisy sorting was also studied in the more general setting of ranking models satisfying a transitivity condition in \cite{shah2017stochastically} and \cite{pananjady2017worst}, where near-optimal minimax rates were derived. Finally, optimal minimax rates for noisy sorting were recently demonstrated in \cite{mao2017minimax}.

There are a number of variations of this problem that have been studied in the literature which we do not survey at length due to space constraints. For example, the papers \cite{yue2012k, szorenyi2015online} considered the online version of this problem with corresponding regret, \cite{chen2015spectral} considered recovering the top $K$ ranked items,   \cite{falahatgar2017maximum, agarwal2017learning, maystre2015just} consider recovering a ranked list of the items,  and \cite{ajtai2016sorting} consider a model where comparisons are not noisy if the item qualities are sufficiently far apart. We refer the reader to the references within those papers for more details on related works in these directions.

\subsection{Our approach} 

We will construct our estimate $\widehat{W}$ by solving a log-least-squares problem described next. We denote by $F_{ij}$ the fraction of times node $i$ wins the comparison against its neighbor $j$, and we further set $R_{ij} = F_{ij}/F_{ji}$. As the number of comparisons on each edge goes to infinity, we will have that $R_{ij}$ approaches $w_i/w_j$ with probability one. Our method consists in finding $\widehat{W}$ as follows: 
\begin{equation}\label{eq:defLS}
\widehat W = \arg\min_{v \in \R_+^{|E|}} \sum_{(i,j)\in E} (\log (v_i/v_j) -\log R_{ij})^2 
\end{equation} 
This can be done efficiently by observing that it amounts to solving the linear system of equations 
$$
\log \widehat W_i - \log \widehat W_j = \log R_{ij}, ~~~~ \mbox{ for all } (i,j) \in E,
$$
in the least square sense. 
Let $B$ to be the incidence matrix\footnote{Given an directed graph with $n$ nodes and $|E|$ edges, the  incidence matrix is the $n \times |E|$ matrix whose $i$'th column has a $1$ corresponding to the source of edge $i$, a $-1$ corresponding to the destination of node $i$, and zeros elsewhere. For an undirected graph, an incidence matrix is obtained by first orienting the edges arbitrarily.} of the comparison graph. Stacking up the $R_{ij}$ into a vector $R$, we can then write
$$
B^T \log \widehat W = \log R
$$
Least-square solutions satisfy
\[ B B^T \log \widehat W = B \log R \] or equivalently
$ L \log \widehat{W} = B \log R$, where $L=B B^T$ is the graph Laplacian. Finally, a solution is given by
\begin{equation}\label{eq:explicit_LS}
\log \widehat W= L^\dag B \log R.
\end{equation} where $L^\dag$ is the Moore-Penrose pseudoinverse. By using the classic results of \cite{spielman2014nearly}, Eq. (\ref{eq:explicit_LS}) can be solved for $\widehat{W}$ to accuracy $\epsilon$ in nearly linear time in terms of the size of the input, specifically in $O(|E| \log^c n \log (1/\epsilon))$ iterations for some constant $c>0$. We note that, for connected graphs, all solutions $w$ of (\ref{eq:defLS}) are equal up to a multiplicative constant and are thus equivalent in terms of  criterion (\ref{eq:error_quad_alpha}).

\subsection{Our contribution}

We will find it useful to view the graph as a circuit with a unit resistor on each edge; $\res_{ij}$ will denote the resistance between nodes $i$ and $j$ in this circuit, $\res_{\rm max}$ denotes the largest of these resistances over all pairs of nodes $i,j=1, \ldots, n$ and similarly $\res_{\rm avg}$ denotes the average resistance over all pairs.  We will use $E_{ij}$ to denote the set of edges lying on at least one simple path starting at $i$ and terminating at $j$, with $E_{\max}$ denoting the largest of the $E_{ij}$. Naturally, $E_{\rm max}$ is upper bounded by the total number of edges in the comparison graph. The performance of our algorithms is described by the following theorem.

\begin{theorem} \label{mainthm}  Let \red{ $\delta \in (0,e^{-1})$}. There exist absolute constants constants $c_1, c_2$ such that,
if  $C_{n,\delta} \geq c_1  \log (n/\delta)$ and $k \geq c_2 E_{\rm max} C_{n,\delta}^2$ and $k \geq c_3 \res b^2 (1 + (\log (1/\delta))$, then
we have, with probability at least \red{$1-\delta$}, that
\begin{align*}
    \sin(\hw,w)^2  \leq   O & \left(   \frac{\min \prt{b^2\res_{\max},b^4\res_{\rm avg} }}{k} \times \right. \\ & \left.~~~~~ \left(    \left(1+\log \frac{1}{\delta} \right)  + \frac{ E_{\max} C_{n,\delta}^2}{k} \right) \right)
\end{align*}
\end{theorem} 

The main feature of this theorem is the favorable form of the bound in the setting when $k$ is large. Then only the leading term  $$\frac{ \min (b^2 \res_{\max}, b^4 \res_{\rm avg}) (1+\log 1/\delta)}{k}$$ dominates the expression on the right-hand-side. Taking square roots, it follows that, asymptotically, 
$$
\abs{\sin(\hw,w)} = \widetilde{O} \left( \sqrt{\frac{ b^2 \res_{\max}}{k}} \right) \textnormal{ and } \widetilde{O} \left( \sqrt{\frac{ b^4 \res_{avg}}{k}} \right),
$$ where the $\widetilde{O}$ notation hides logarithmic factor in $\delta$. 

Our other main result is that, in the regime when $k$ is large, there is very little room for improvement.
\begin{theorem}\label{thm:lower_resistance}
For any comparison graph $G$, and for any algorithm, as long as $k \geq c \sqrt{\lambda_{\rm max}(L)} n \Omega_{\rm avg}$ for some absolute constant $c$, we have that 
$$
\sup_{w \in \R_+^n} E \left| \sin(\hw,w) \right| \geq \Omega\prt{\sqrt{\frac{\res_{avg}}{k}}},
$$ where as before $L$ is the graph Laplacian. 
\end{theorem}

Comparing Theorem \ref{mainthm} with Theorem \ref{thm:lower_resistance}, we see that the performance bounds of Theorem \ref{mainthm} are minimax optimal, at least up to the logarithmic factor in the confidence level $\delta$ and dependence on the skewness factor $b$. {\em We can thus conclude that the square root of the graph resistance is the key graph-theoretic property which captures how relative error decays for learning from pairwise comparisons.} This observation is the main contribution of this paper.

\subsection{Comparison to previous work}

Table \ref{tab:comp} quantifies how much the bound of Theorem \ref{mainthm} expressed in terms of $\res_{\max}$  improves the asymptotic decay rate on various graphs over the bound \cite{negahban2016rank}. The $\widetilde{O}$ notation ignores log-factors. Both random graphs are taken at a constant multiple threshold which guarantees connectivity; for Erdos-Renyi this means $p=O((\log n)/n)$ and for a geometric random graph, this means connecting nodes at random positions at the unit square when they are $O \left( \sqrt{(\log n)/n} \right)$
apart. 

\begin{table}[t]
\caption{Comparison, for different families of graphs, of   
$\widetilde  O \left( \frac{d_{\rm max}}{d_{\rm min}(1-\lambda)} \sqrt{\frac{1}{d_{\rm max}} } \right)$ and $\widetilde O(\sqrt{b \res_{\max}  })$,
which are, respectively, the asymptotic bounds  (\ref{eq:bound_negahban}) in \cite{negahban2016rank}, and the first bound from our Theorem \ref{mainthm}. The common decay in $k^{-1/2}$ is omitted for the sake of conciseness.}
\label{tab:comp}
\begin{center}
\begin{tabular}{| c |c | c |}
\hline 
Graph & Eq. (\ref{eq:bound_negahban}) & Theorem \ref{mainthm}
\\
\hline
Line & $b^{5/2} n^2$ & $b \sqrt{n}$ \\ 
Circle & $b^{5/2} n^2$ & $b \sqrt{n}$ \\ 
2D grid & $b^{5/2} n$ & $b$ \\
3D grid & $b^{5/2} n^{2/3}$ & $b$ \\ 
Star graph & $b^{5/2} \sqrt{n} $ & $b$ \\ 
2 stars joined at centers & $b^{5/2} n^{1.5}$ & $b$ \\
Barbell graph  & $b^{5/2} n^{3.5}$ & $b \sqrt{n}$ \\
Geo. random graph & $b^{5/2} n$ & $b$ \\
Erdos-Renyi  & $b^{5/2} $ & $b$ \\
\hline
\end{tabular}
\end{center}
\end{table}

\medskip

Most of the scalings for eigenvalues of normalized Laplacians used in Table \ref{tab:comp} are either known or easy to derive. For an analysis of the eigenvalue of the barbell graph\footnote{Following \cite{wilf1989editor}, the barbell graph refers to two complete graphs on $n/3$ vertices connected by a line of $n/3$ vertices.}, we refer the reader to \cite{landau1981bounds}; for mixing times on the geometric random graph, we refer the reader to \cite{avin2007cover}; for the resistance of an Erdos-Renyi graph, we refer the reader to \cite{sylvester2016random}. 

In terms of the worst-case performance in terms of the number of nodes, our bound grows at worst as $\widetilde{O} \left( b\sqrt{n/k}\right)$ using the observation that $\res_{\max} = O(n)$.  By contrast, for the barbell graph, the bound of \cite{negahban2016rank} grows as $\widetilde{O}(b^{5/2} n^{3.5}/\sqrt{k})$, and it is not hard to see this is actually the worst-case scaling in terms of the number of nodes. 

Finally, we note that these comparisons use slightly different error measures: $|\sin(\hw, w)|$ on our end vs the relative error in the $2$-norm after $w,\widehat{W}$ have been normalized to sum to one, used by \cite{negahban2016rank}. To compare both in terms of the latter, we could multiply our bounds by $\sqrt{b}$ (see Lemma \ref{lem:compare_criteria}).

\subsection{Notation} The remainder of this paper is dedicated to the proof Theorem \ref{mainthm} (Theorem \ref{thm:lower_resistance} is proved in the Supplementary Information). However, we first collect some notation we will find occasion to use.

As mentioned earlier, we let $F_{ij}$ be the empirical rate of success of item $i$ in the $k$ comparisons between $i$ and $j$; thus $E [F_{ij}] = p_{ij}$ so that the previously introduced $R_{ij}$ can be expressed as $R_{ij} = \frac{F_{ij}}{F_{ji}}$. We also let $\rho_{ij}= w_i/w_j = p_{ij}/p_{ji}$, to which $R_{ij}$ should converge asymptotically.

We will make a habit of stacking any of the quantities defined into vectors; thus $F$, for example, denotes the vector in $\R^{|E|}$ which stacks up the quantities $F_{ij}$ with the choice of $i$ and $j$ consistent with the orientation in the incidence matrix $B$. The the vectors $p$ and $\rho$ are defined likewise.

\section{Proof of the algorithm performance (Theorem \ref{mainthm})}

We begin the proof with a sequence of lemmas which work their way to the main theorem. The first step is to introduce some notation for the comparison on the edge $(i,j)$.

Let $X_{ij}$ be the outcome of a single coin toss comparing coins $i$ and $j$. Using the standard formula for the variance of a Bernoulli random variable, we obtain
\begin{eqnarray}
    {\rm Var}(X_{ij}) &= p_{ij} (1-p_{ij}) = \frac{w_i w_j}{(w_i + w_j)^2}\nonumber \\& = \frac{1}{\rho_{ij} + 2 + \rho_{ij}^{-1}} =: \frac{1}{v_{ij}}, \label{eq:variance_Xij}
\end{eqnarray}
where we have defined $v_{ij} = \rho_{ij} + 2 + \rho_{ij}^{-1}$. Observe that $v_{ij}$ is always upper bounded by $3+\max(\rho_{ij},\rho_{ji}) \leq 3 +b \leq 4b$, where we remind $b\geq \max_{i,j}\frac{w_i}{w_j}$. 

We first argue that all $F_{ij}$ are reasonably close to their expected values. For the sake of concision, we state the following assumptions about the constants, $\delta$, $k$ and the quantity $C_{n,\delta}$. Note that some of the intermediate results hold under weaker assumptions, but we omit these details for the sake of simplicity. 
\begin{assumption}\label{assump:ratio_constant}
We have that $\delta \leq e^{-1}$, 
$C_{n,\delta} \geq c_1  \log (n/\delta)$, 
and $k \geq c_2  b  (C_{n,\delta}+1) \max \{ \res_{\max}, E_{\max} \} $.
\end{assumption}

The following lemma is a standard application of Chernoff's inequality. For completeness, a proof is included in Section \ref{highprobproof} of the Supplementary Information. 

\begin{lemma}  
There exist absolute constants constants $c_1, c_2$ such that, under Assumption \ref{assump:ratio_constant}, we have
\[ P \left( \max_{(i,j) \in E} \left| F_{ij} -  p_{ij} \right|  \geq  \sqrt{\frac{C_{n,\delta}}{k v_{ij}}} \right) \leq \delta. \]  
 \label{highprob}
\end{lemma}

The next lemma provides a convenient expression for the quantity $\log \widehat{W} - \log w$ in terms of the ``measurement errors'' $F-p$. 
Note that the normalization assumption is not a loss of generality since $w$ is defined up to a multiplicative constant, and is directly satisfied if $\hw$ is obtained from (\ref{eq:explicit_LS}).

\providecommand{\dv}{V}
\begin{lemma} Suppose $w$ is normalized so that $\sum_{i=1}^n \log w_i = 0$. There exist absolute constants $c_1, c_2 > 0$ such that, under Assumption \ref{assump:ratio_constant},  
there holds with probability $1-\delta$
\begin{equation} \label{eq:decomp} \log \widehat{W} - \log w  = L^\dag B \dv(F-p) + L^\dag B
\Delta,  
\end{equation}  and 
\begin{equation} \label{eq:infbound} ||\Delta||_{\infty} \leq O \left( \frac{b C_{n,\delta}}{k} \right),
\end{equation}
where 
$\dv$ is a $\abs{E}\times \abs{E}$ diagonal matrix whose entries are the $v_{ij}$, for all edges $(i,j)\in E$.
\label{lemma:decomp}
\end{lemma} 
\begin{proof} 
By definition 
\[ \log w_i - \log w_j = \log \rho_{ij} \mbox{ for all } (i,j) \in E,\] which we can write as $B^T \log w = \log \rho.$ It follows that 
\[ \log w = (B B^T)^\dag B \log \rho = L^\dag B \log \rho, \] since $w$ is assumed normalized so that $\sum_{i=1}^n \log w_i = 0$. Combining this with Eq. (\ref{eq:explicit_LS}), we obtain 
\begin{equation} \label{eq:laplinv} \log \widehat W - \log w = L^\dag B (\log R - \log \rho). 
\end{equation}

We thus turn our attention to analyzing the vector $\log R - \log \rho$. Our analysis will  be conditioning on the event that for all $(i,j) \in E,$
\begin{equation} \label{eq:assump}  \{ \left| F_{ij} -  p_{ij} \right|  \leq  \sqrt{\frac{C_{n,\delta}}{k v_{ij}}} \}, 
\end{equation}
which, by Lemma \ref{highprob}, holds with probability at least $1-\delta$. We will call this event $\mathcal{A}$. 

We begin with one implication that comes from putting together event $\mathcal{A}$ and our assumption $k \geq c_1 b C_{n,\delta}$ (in Assumption \ref{assump:ratio_constant}) for a constant $c_1$ that we can choose: that we can assume that \begin{equation} \label{eq:fifth} \max_{(i,j) \in E} |F_{ij} - p_{ij} | \leq \frac{\min (p_{ij},p_{ji})}{5}. \end{equation} Indeed, from Eq. (\ref{eq:assump}) for this last equation to hold it suffices to have $k \geq 25 C_{n,\delta}/(v_{ij} p_{ij}^2) \mbox{ for all } (i,j) \in E.$ Observing that \[
\frac{1}{v_{ij} p_{ij}^2} = \frac{1}{\frac{1}{p_{ij} p_{ji}} p_{ij}^2} = \rho_{ji} \leq b,
\] we see that assuming $k \geq 25 b C_{n,\delta}$ is sufficient for Eq. (\ref{eq:fifth}) to hold conditional on event $\mathcal{A}$.

Our analysis of $\log R - \log \rho$ begins with the observation that since 
\[ R_{ij} = \frac{1-F_{ji}}{F_{ji}}, ~~ \rho_{ij} = \frac{1-p_{ji}}{p_{ji}} \] we have that
\begin{eqnarray*} \log R_{ij} - \log \rho_{ij} & = & \log \left( \frac{1}{F_{ji}} - 1 \right) - \log \left(  \frac{1}{p_{ji}} - 1 \right) 
\end{eqnarray*} Next we use Taylor's expansion of the function $g(x) = \log (1/x - 1)$, for which we have 
\[ g'(x)  =  \frac{1}{x(x-1)},~~ g'(p_{ji})  =  - v_{ij}, ~~g''(x)  =  \frac{1-2x}{x^2 (1-x)^2}
\]
to obtain that $\log R_{ij} - \log \rho_{ij} $ can thus be expressed as
\begin{equation} \label{eq:R-rho_ij}
- v_{ij} (F_{ji} - p_{ji}) + \frac{1}{2} \frac{1-2z_{ji}}{z_{ji}^2 (1-z_{ji})^2} (F_{ij} - p_{ij})^2 
\end{equation}
where $z_{ji}$ lies between $p_{ji}$ and $F_{ji}$ (and $1-z_{ji}$ lies thus between $p_{ij}$ and $F_{ij}$). We can rewrite this equality in a condensed form
\begin{equation}\label{eq:R-rho_condensed} \log R - \log \rho = \dv(F-p) + \Delta, \end{equation}
where $\Delta$ corresponds to the second terms in (\ref{eq:R-rho_ij}), which we will now bound. 
Because we have conditioned on event $\mathcal{A}$, which, as discussed above implies $|F_{ji} - p_{ji}| \leq \min(p_{ji},p_{ji})/5$, we actually have that 
$z_{ji} \in [0.8 p_{ji}, 1.2 p_{ji}]$ and that $1-z_{ji}$ lying between $p_{ij}$ and $F_{ij}$ belongs to  $[0.8 p_{ij},1.2p_{ij}]$. Hence
\[ 
|\Delta_{ij}| \leq \frac{1}{2} \frac{1}{0.8^4 p_{ij}^2 p_{ji}^2} (F_{ij} - p_{ij})^2 
\leq  c_3 v_{ij} \frac{C_{n,\delta}}{k},
\] 
for  $c_3= \frac{1}{2\times(0.8)^4}$, and where we have used (\ref{eq:assump}) for the last inequality. Plugging this into Eq. (\ref{eq:R-rho_condensed}) and (\ref{eq:laplinv}) completes the proof, and  Eq. (\ref{eq:infbound}) follows from the last equation combined with the fact that  $v_{ij} \leq 4b$ for all $(i,j) \in E$. 
\end{proof} 

The following lemma bounds how much the {\em ratios} of our estimates $\widehat{W}_l$ differ from the corresponding ratios of the true weights $w_l$. To state it, we will use the notation \[ Q_{ij} = (\e_i - \e_j) (\e_i - \e_j)^T, \] where $\e_i$ is the standard notation for the $i$'th basis vector. Furthermore, we define the  product 
\begin{equation}\label{eq:def_product_ij}
    \langle x, y \rangle_{(i,j)} = x^T B^T L^\dag Q_{ij} L^\dag B y, ~~~~ ||x||_{(i,j)}^2 = \langle x, x \rangle_{(i,j)}. 
\end{equation}  Observe that the matrix $B^T L^\dag Q_{ij} L^\dag B$  is  positive semidefinite, which implies by standard arguments that
\[ \langle x+y, x+y \rangle_{(i,j)} \leq 2 \langle x,x \rangle_{(i,j)} + 2 \langle y, y\rangle_{(i,j)} \] holds for all vectors $x,y$.

\begin{lemma} Suppose $w$ is normalized so that $\sum_{i=1}^n \log w_i = 0$. There exist absolute constants $c_1, c_2 > 0$ such, under Assumption \ref{assump:ratio_constant}, with 
with probability $1-\delta$, we have that for all pairs $i,j=1, \ldots, n$,  
\begin{equation} \label{eq:comparison} \left( \log \frac{\widehat W_i}{\widehat W_j} - \log \frac{w_i}{w_j} \right)^2 \leq 2 \left| \left|\dv(F-p) \right| \right|_{(i,j)}^2 + 2 \left| \left | \Delta \right| \right|_{(i,j)}^2,
\end{equation}  and 
\[ ||\Delta||_{\infty} \leq O \left( \frac{b C_{n,\delta}}{k} \right).
\] \label{lemma:comparison}
\end{lemma} 
\begin{proof} Observe that, on the one hand,  using Lemma \ref{lemma:decomp},
\begin{small}\begin{equation}\begin{array}{l} 
(\log \widehat{W} - \log w)^T  Q_{ij} (\log \widehat{W} - \log w)  \\
=  \left( L^\dag B \dv(F-p) + L^\dag B \Delta \right)^T Q_{ij} \left( L^\dag B \dv(F-p) + L^\dag B \Delta \right)   \\ 
 =  \langle \dv(F-p) + \Delta, \dv(F-p) + \Delta \rangle_{(i,j)}  \\ 
 \leq  2 \langle \dv(F-p), \dv(F-p) \rangle_{(i,j)} + 2  \langle \Delta, \Delta \rangle_{(i,j)}\label{eq:mainsum}  
\end{array}\end{equation}\end{small}
which is the right-hand side of (\ref{eq:comparison}). On the other hand, observe that 
\begin{small}\begin{equation}\begin{array}{l} 
\left( \log \widehat{W} - \log w \right) Q_{ij}   \left( \log \widehat{W} - \log w \right)\\
= \left( \log \widehat{W}_i - \log w_i - \left( \log \widehat{W}_j - \log w_j \right) \right)^2\\ 
=  \left( \log \frac{\widehat W_i}{\widehat W_j} - \log \frac{w_i}{w_j} \right)^2 \label{eq:squaredexpr}
\end{array}\end{equation}\end{small}
Combining Eq. (\ref{eq:mainsum}) with Eq. (\ref{eq:squaredexpr}) completes the proof. 
\end{proof}

Having proved Lemma \ref{lemma:comparison}, we now analyze each of the terms in the right-hand side of Eq. (\ref{eq:comparison}). We begin with the second term, i.e., with $||\Delta||_{(i,j)}^2$. To bound it, we will need the following inequality.

\begin{lemma}  \label{lemma:current} For any $\Delta \in \R^{|E|}$, we have that \[ |\Delta^T B^T L^\dag ({\bf e}_i - {\bf e}_j)| \leq ||\Delta||_{\infty} \sqrt{\res_{ij} |E_{ij}|} , 
\] where, recall, $\res_{ij}$ is the resistance between nodes $i$ and $j$, and $E_{ij}$ is the set of edges belonging to some simple path from $i$ to $j$. 
\end{lemma} 

\begin{proof} The result follows  from circuit theory, and we sketch it out along with the relevant references. The key idea is that the vector $u = B^T L^\dag ({\bf e}_i - {\bf e}_j)$ has a simple electric interpretation. We have that  $u \in \R^{|E|}$ and the $k$'th entry of $u$ is the current on edge $k$ when a unit of current is put into node $u$ at removed at node $j$. For details, see the discussion in Section 4.1 of \cite{vishnoi2013lx}.

This lemma  follows from several consequences of this interpretation. First, the entries of $u$ are an acyclic flow from $i$ to $j$; this follows, for example, from Thompson's principle which asserts that the current flow minimizes energy (see Theorem 4.8 of \cite{vishnoi2013lx}). Moreover, Thompson's principle further asserts that  $\res_{ij} = ||u||_2^2$. Finally, by the flow decomposition theorem (Theorem 3.5 in \cite{ahuja2017network}), we can decompose this flow along simple paths from $i$ to $j$; this implies that $|{\rm supp}(u)| \leq |E_{ij}|$.

With these facts in mind, we apply Cauchy-Schwarz to obtain   
\[ ||u||_1 \leq ||u||_2 \sqrt{|{\rm supp}(u)|} \leq \sqrt{\res_{ij} |E_{ij}|},
\] and then conclude the proof using Holder's inequality \begin{small}
\[ |\Delta^T B^T L^\dag (\e_i - \e_j)| = |\Delta^T u| \leq ||\Delta||_{\infty} ||u||_1  ||\Delta||_{\infty} \sqrt{\res_{ij} |E_{ij}|}. \] \end{small}\end{proof} 
As a corollary, we are able to bound the second term in Eq. (\ref{eq:comparison}). The proof follows immediately by combining Lemma \ref{lemma:current} with Lemma \ref{lemma:comparison}. 

\begin{corollary}\label{cor:secterm}
There exist absolute constants $c_1, c_2 > 0$ such that,
under Assumption \ref{assump:ratio_constant},
with probability $1-\delta$, we have that for all pairs $i,j=1, \ldots, n$,  \[ \left| \left| \Delta \right| \right|_{(i,j)}^2  \leq O \left( \res_{ij} E_{ij} \frac{b^2 C_{n,\delta}^2}{k^2} \right). \]\end{corollary} 

We now turn to the first-term in Eq. (\ref{eq:comparison}), which is bounded in the next lemma. 

\begin{lemma} There exist absolute constants $c_1, c_2$ such that,
under Assumption \ref{assump:ratio_constant},  
with probability $1-\delta$ we have that for all pairs $i,j=1, \ldots,n$, 
\[   \left| \left| \dv(F-p) \right| \right|_{(i,j)}^2 \leq O \left(  \res_{ij}   \frac{b^2  }{k} \left(1+\log \frac{1}{\delta} \right) \right)   \]  \label{lem:firstterm}
\end{lemma} 

\begin{proof}
 The random variable $X_{ij} - p_{ij}$ (where, recall, $X_{ij}$ is the outcome of a single comparison between nodes $i$ and $j$) is zero-mean and supported on an interval of length $1$, and consequently it is subgaussian\footnote{A random variable $Y$ is said to be subgaussian with parameter $\tau$ if $E[e^{\lambda Y}] \leq e^{\tau^2 \lambda^2/2}$ for all $\lambda$.} with parameter $1$ (see \red{Section 5.3} of \cite{lattimore2018bandit}). By standard properties of subgaussian random variables, it follows that \red{$v_{ij}(F_{ij} - p_{ij})$} is subgaussian with $\tau = v_{ij}/\sqrt{k} \leq 4b/\sqrt{k}$. It follows then from  \red{Theorem 2.1} of \cite{hsu2012tail} for subgaussian random variables applied to $||(e_i-e_j) B^TL^{\dag} (F-p)||^2 =  \left| \left| \dv(F-p) \right| \right|_{(i,j)}^2$, that for any $t\geq 1$ there is a probability at least $1-e^{-t}$ that \begin{footnotesize}
\begin{eqnarray*}
\left| \left| \dv(F-p) \right| \right|_{(i,j)}^2  
&  \leq & \frac{16b^2}{k}\prt{{\rm tr}(M) + 2\sqrt{{\rm tr} (M^2)t} +2\norm{M}t} \nonumber \\
&   \leq &  \frac{16b^2}{k} {\rm tr}(M) (1+ 4t),\label{eq:ineq_trM14t}
\end{eqnarray*} \end{footnotesize}
where we have used $\sqrt{t}\leq t$, ${\rm tr} (M^2) \leq {\rm tr} (M)^2$ and $\norm{M}\leq  {\rm tr} (M)$.
We now compute this trace. 
\begin{eqnarray} {\rm tr}(M) & = & {\rm tr}(B^T L^\dag Q_{ij} L^\dag B) \nonumber \\ & = & {\rm tr}( Q_{ij} L^\dag B B^T L^\dag) \nonumber  
 =  {\rm tr}(Q_{ij} L^\dag ) \nonumber \\ 
& = & (\e_i - \e_j)^T L^\dag (\e_i - \e_j) 
 =  \res_{ij},\label{eq:tracecalc}
\end{eqnarray} where the  second equality uses the well-known property of the Moore-Penrose pseudo-inverse: $A^\dag A A^\dag = A^\dag$ for any matrix $A$ (see Section 2.9 of \cite{drineas2018lectures}); and last equality uses a well-known relation between resistances and Laplacian pseudoinverses, see Chapter 4 of \cite{vishnoi2013lx}. The result follows then from the application of (\ref{eq:ineq_trM14t}) to $t=\log{1/\delta}$.
\end{proof}

Having obtained the bounds in the preceding sequence of lemmas, we now return to Lemma \ref{lemma:comparison} and ``plug in'' the results we have obtained. The result is the following lemma. 

\begin{lemma}\label{lemma:puttingtogether} There exist absolute constants $c_1, c_2 > 0$ such, 
under Assumption \ref{assump:ratio_constant}, 
with probability $1-\delta$, we have that for all pairs $i,j=1, \ldots, n$, \begin{footnotesize} \begin{eqnarray*}
&\left[{\frac{\widehat W_i}{\widehat W_j} - \rho_{ij}} \right]^2 \leq O \left( \rho_{ij}^2 \frac{b \res_{ij}}{k}   \left(b  (1+\log (1/\delta))  + \frac{b E_{i,j} C_{n,\delta}^2}{k} \right) \right)\end{eqnarray*} \end{footnotesize}
\end{lemma}

\begin{proof} 
By putting together Lemma \ref{lemma:comparison} with Corollary \ref{cor:secterm} and Lemma \ref{lem:firstterm}, we obtain that, with probability at least $1-\delta$,\begin{equation}\label{eq:baseeq}\begin{array}{l}\left( \log \frac{\widehat W_i}{\widehat W_j} - \log \frac{w_i}{w_j} \right)^2 \\ \leq   O \left( \frac{b \res_{ij}}{k}   \left(   b  (1+\log (1/\delta))  + \frac{b E_{i,j} C_{n,\delta}^2}{k} \right)  \right)
\end{array}\end{equation} 
\red{Observe that for a sufficiently large $c_2$, if $k\geq c_2E_{ij}C_{n,\delta}^2$ then the term $b  (1+\log (1/\delta))  + \frac{b E_{i,j} C_{n,\delta}^2}{k}$ is bounded by $O(b(1 + \log (1/\delta)))$. Hence, if $k$ is also at least $c_2 b^2 \res_{ij} (1+\log(1/\delta))$ (which holds due to Assumption \ref{assump:ratio_constant}), equation (\ref{eq:baseeq}) implies} 
\begin{equation}\label{eq:bound_abs_dif}
 \abs{\log \frac{\widehat{W}_i}{\widehat{W}_j} - \log \frac{w_i}{ w_j}} \leq 1.
 \end{equation}
A particular implication is that $\max\left(e^{\log(\widehat{W}_i/\widehat{W}_j)},\right.$ $\left. e^{\log (w_i/w_j)}\right)$ $\leq  e^{1+\log(w_i/w_j)}$.  Applying the inequality  $|e^{a} - e^b| \leq \max\{e^a, e^b\} |a-b|$ to (\ref{eq:bound_abs_dif}) leads then to
$$\left| \frac{\widehat W_i}{\widehat W_j} - \frac{w_i}{w_j} \right| 
 \leq e^{1+\log (w_i/w_j)} \left| \log \frac{\widehat{W}_i}{\widehat{W}_j} - \log \frac{w_i}{ w_j} \right|
$$ and now using $e^{\log(w_i/w_j)} = \rho_{ij}$, the proof follows by combining the last equation with Eq. (\ref{eq:baseeq}).
\end{proof}

The next lemma demonstrates how to convert Lemma \ref{lemma:puttingtogether} into a bound on the relative error between $\widehat{W}$ and the true weight vector $w$.

\begin{lemma}\label{lem:ratio->criterion}
Suppose we have that 
$$
\left[ {\frac{\widehat W_i}{\widehat W_j} - \rho_{ij}} \right]^2 \leq \rho_{ij}^2 s_{ij}(k),
$$
for all $i,j=1, \ldots, n$. Fix index $\ell \in \{1, \ldots, n\}$. Then there hold
\begin{eqnarray}\label{eq:bound_s*} \sin(w,\hw) &\leq&  \max_j s_{j\ell}(k),\\ 
\label{avgbound}\sin(w,\hw) &\leq&  b^2 s_{\rm avg},\end{eqnarray}
where $s_{\rm avg} = \frac{\sum_{a,b = 1, \ldots, n} s_{ab} }{n^2}$.
\end{lemma}

\begin{proof}
It follows from Lemma \ref{lem:sinus} that for all $\alpha$,
$$
\sin(w,\hw) \leq \frac{||\widehat W - \alpha w ||_2^2}{||\alpha w||_2^2}. 
$$
Taking $\alpha = \widehat W_\ell/w_\ell$, we get
\[ 
 \frac{||\widehat W - \alpha w ||_2^2}{|| \alpha w||_2^2}  = 
\frac{\sum_i (\widehat W_i- \frac{\widehat W_\ell}{w_\ell} w_i)^2}{\sum_i \frac{\widehat W_\ell^2}{w_\ell^2} w_i^2} \\
=
 \frac{\sum_i (\frac{\widehat W_i}{\widehat W_\ell}-\rho_{i\ell})^2}{ \sum_i \rho_{i\ell}^2}.  
\]
Using the assumption of this lemma, we obtain
\begin{equation}\label{eq:bound_sio(k)}
 \frac{||\widehat W - \alpha w ||_2^2}{|| \alpha w||_2^2} \leq  \frac{\sum_i s_{i\ell}(k)\rho_{il}^2}{ \sum_i \rho_{i\ell}^2}, 
\end{equation}
from which (\ref{eq:bound_s*}) follows. Another consequence of (\ref{eq:bound_sio(k)}) is that
\begin{equation}\label{eq:pre_bound_avg}
 \frac{||\widehat W - \alpha w ||_2^2}{|| \alpha w||_2^2} \leq  \frac{(\max_i \rho_{i\ell}^2)\sum_i s_{i\ell}(k)}{ n \min_j \rho_{j\ell}^2} \leq b^2 \frac{\sum_{i=1}^n s_{i\ell}}{n},
\end{equation} 
where we used  
$$
\red{\frac{\max_i \rho_{i\ell}}{\min_i \rho_{i\ell} }= \frac{\max_i w_i/w_\ell }{\min_j w_j/w_\ell} = \max_{i,j}\frac{w_i}{w_j}\leq b.}
$$
Observe now that since $s_{\rm avg} = \frac{1}{n}\sum_\ell \frac{1}{n}\sum_i s_{i\ell} $, there must exist at least one $\ell$ for which $\sum_{i=1}^n s_{i\ell}\leq s_{\rm avg}$. Hence (\ref{avgbound}) follows from (\ref{eq:pre_bound_avg}).
\end{proof}

Having proven this last lemma, Theorem \ref{mainthm} follows immediately by combining By Lemma \ref{lemma:puttingtogether} and Lemma \ref{lem:ratio->criterion}.

\section{Experiments}\label{sec:expe}

\providecommand{\Ntest}{N_{\rm test}}

The purpose of this section is two-fold. First, we would like to demonstrate that simulations are consistent with Theorem \ref{mainthm}; in particular, we would like to see error scalings that are consistent with the average resistance, rather than e.g., spectral gap.  Second, we wish observe that, although our results are asymptotic, in practice the scaling with resistance appears immediately, even for small $k$. Since our main contribution is theoretical, and since we do not claim that our algorithm is better than available methods in practice, we do not perform a comparison to other methods in the literature. Additional details about our experiments are provided in \red{Section \ref{sec:info_expe} in the Supplementary Information.}

We begin with Erdos-Renyi comparison graphs. Figure \ref{fig:ER_k} shows the evolution of the error with the number $k$ of comparisons per edge. The error decreases as $O(1/\sqrt{k})$ as predicted. Moreover, this is already the case for small values of $k$.

Next we move to the influence of the graph properties. Figure \ref{fig:ER_prog_graph} shows that the average error is asymptotically constant when $n$ grows while keeping the expected degree $d:=(n-1)p$ constant, and that it decreases as $O(1/\sqrt{d})$ when the expected degree grows while keeping $n$ constant. This is consistent with our analysis in Table \ref{tab:comp}, and with the results \cite{boumal2014concentration} showing that the average resistance $\res_{\rm avg}$ of Erdos-Renyi graphs evolves as $O(1/d)$.

We next consider lattice graphs 
in Figure \ref{fig:lattice}. For the 3D lattice, the error appears to converge to a constant when $n$ grows, which is consistent with our results since the average resistance of 3D lattice is bounded independently of $n$. The trend for the 2D lattice appears also 
consistent with a bound in $O(\sqrt{\log n})$ predicted by our results since the resistance on 2D lattice evolves as $O(\log n)$.

\begin{figure}[h!]
    \centering
    \includegraphics[scale = .25]{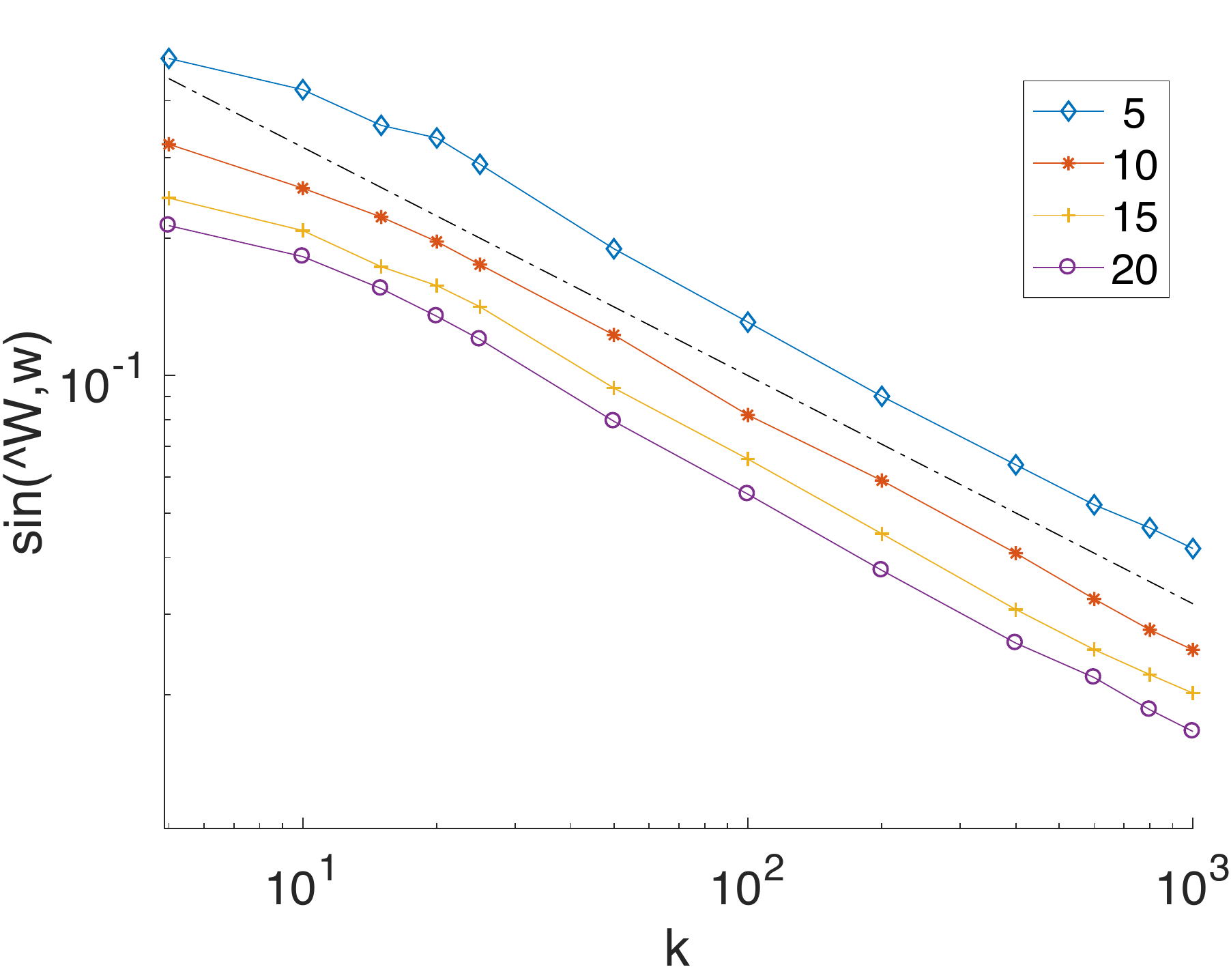}
    \caption{Error evolution with the number $k$ of comparisons per edge in Erdos-Renyi graphs of 100 nodes, for different expected degrees $d=(n-1)p$, with $b=10$. Each line corresponds to a different expected degree. The results are averaged over $\Ntest=100$ tests. The dashed line is proportional to $1/\sqrt{k}$.}
    \label{fig:ER_k}
\end{figure}
\begin{figure}[h!]
    \centering
    \begin{tabular}{cc}
\hspace{-.5cm}    \includegraphics[scale = .22]{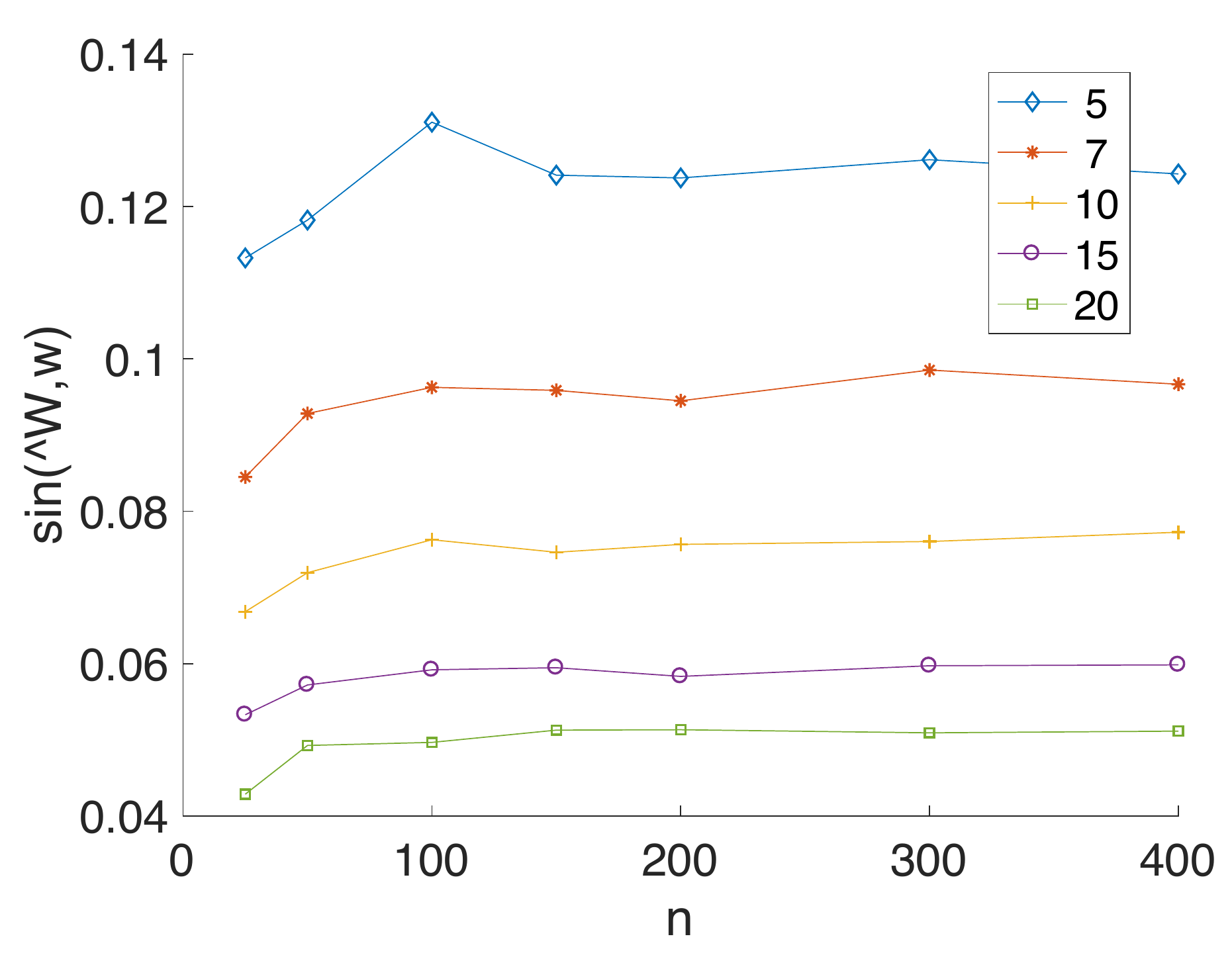}& \hspace{-.2cm} \includegraphics[scale = .22]{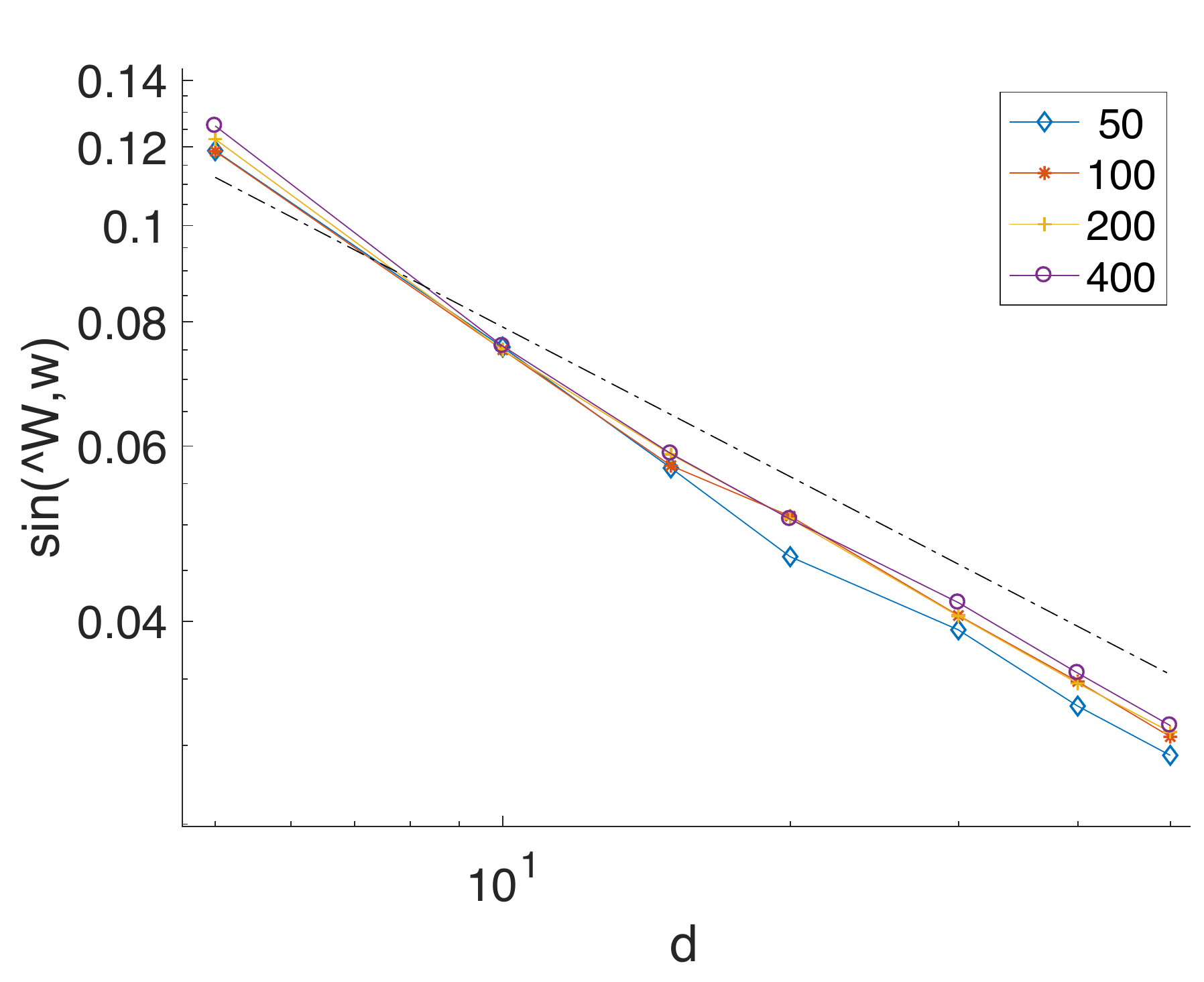}\\
(a) & (b)    
    \end{tabular}    \vspace{-.3cm}
    \caption{Error evolution with the number of nodes $n$ for different expected degrees $d=(n-1)p$ (a), and with the expected degree $(n-1)p$ for different number of nodes $n$ (b). There are $k=100$ comparisons per edge, $b=5$, and results are averaged over $\Ntest=50$ tests. The dashed line in (b) is proportional to $1/\sqrt{k}$.}
    \label{fig:ER_prog_graph}
\end{figure}
\begin{figure}[h!]
    \centering
    \begin{tabular}{cc}
    \hspace{-.5cm}\includegraphics[scale = .22]{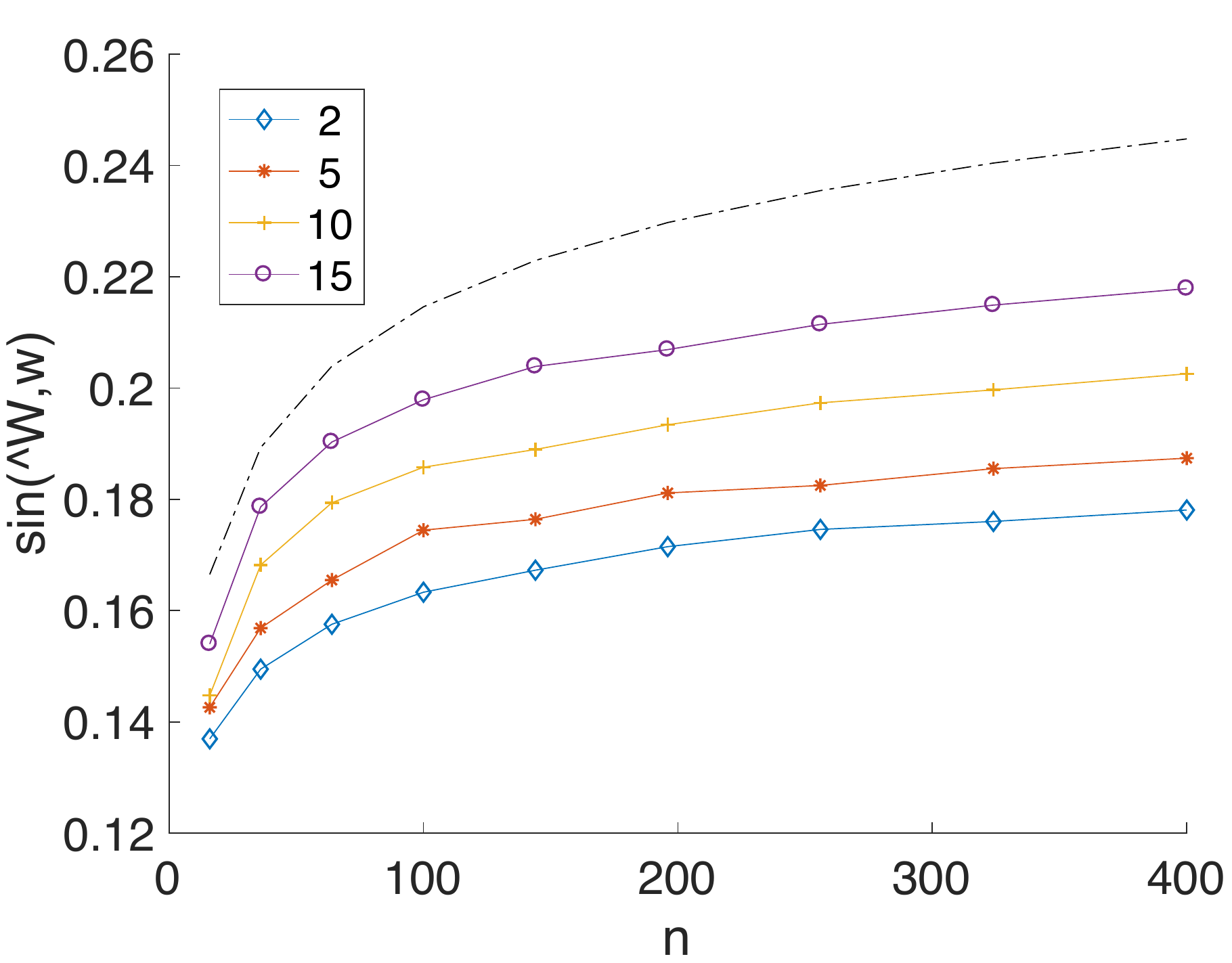}&      
    \hspace{-.2cm}\includegraphics[scale = .22]{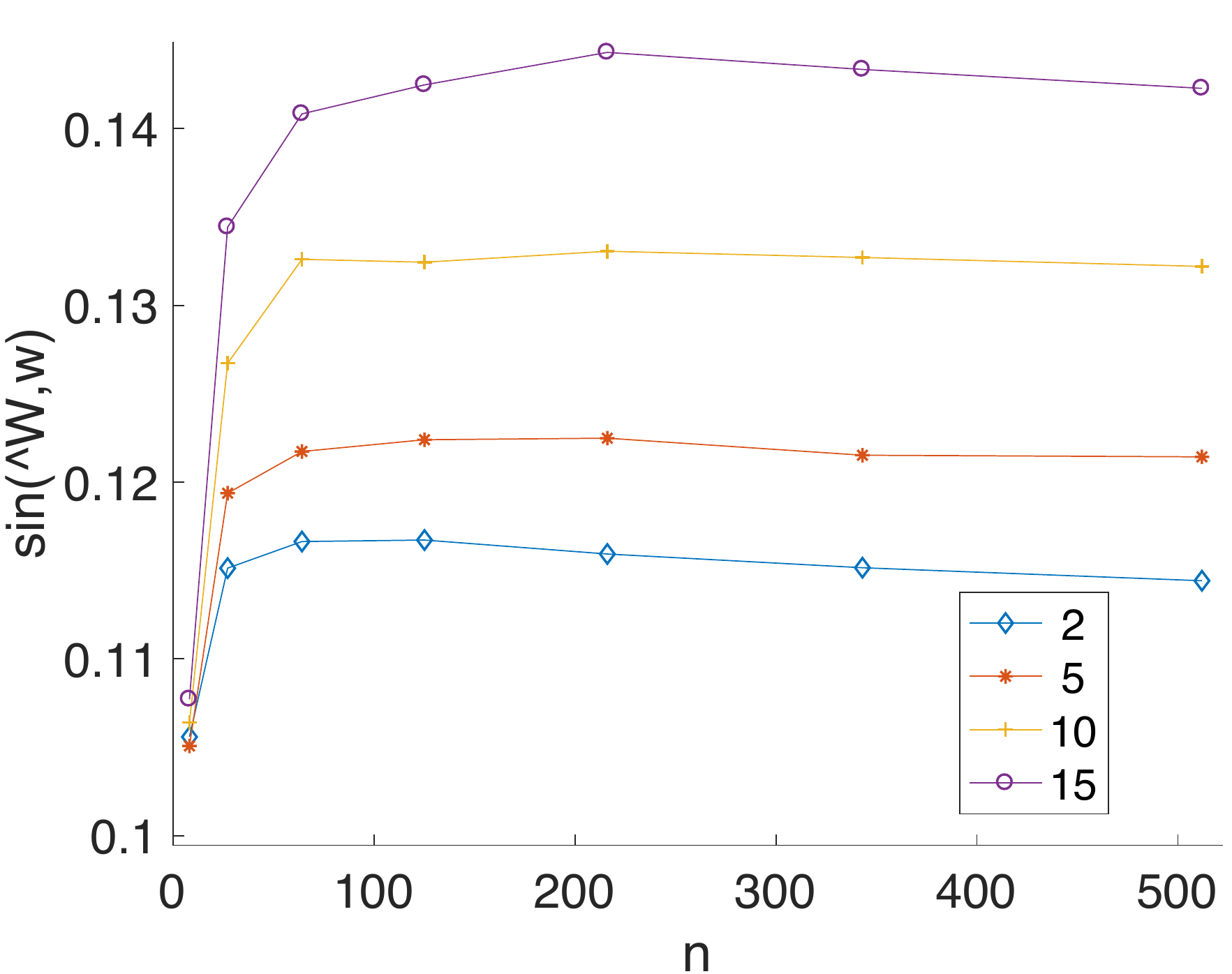}\\(a) &(b)    
    \end{tabular}
    \vspace{-.3cm}
    \caption{Error evolution with the number of nodes for regular lattices in 2D (a), and 3D (b). Each line corresponds to a different choice of $b$. The number of comparisons is $k=100$ per edge. Results are averaged over respectively $\Ntest =1000$ and $\Ntest = 2000$ tests. The dashed line in (a) is proportional to $\sqrt{\log n}$.}
    \label{fig:lattice}
%
\end{figure}

\section{Conclusion}
Our main contribution has been to demonstrate, by a combination of upper and lower bounds, that the error in quality estimation from pairwise comparisons scales as the graph resistance. Our work motivates a number of open questions.

First, our upper and lower bounds are not tight with respect to skewness measure $b$. We  conjecture that the scaling of $\widetilde{O}(\sqrt{b \res_{\rm avg}/k})$ for relative error is optimal, but either upper of lower bounds matching this quantity are currently unknown.

Second, it would interesting to obtain non-asymptotic version of the results presented here. 
Our simulations are consistent with the asymptotic scaling $\widetilde{O}(\sqrt{\res_{\rm avg}/k})$ (ignoring the dependence on $b$)
being effective immediately, but at the moment we can only prove this scaling governs the behavior as $k \rightarrow \infty$.

\bibliographystyle{style_files/icml2019}
\bibliography{comp-bib}

\newpage

$~$

\newpage

\appendix

\section{Supplementary Information: relation between different relative error criteria}\label{sec:criteria}

Our relative error criterion of  $|\sin(\hw,w)|$ differs somewhat from the criterion used in \cite{negahban2016rank}, which was
\[ \frac{||\widehat{W} - w||_2}{||w||_2},\] where both $w$ and and $\widehat{W}$ need to be normalized to sum to $1$. To represent this compactly, we introduce the notation $D(x,y)$ for positive vectors $x,y$, defined as 
$$D(x,y)=  \frac{\norm{\frac{y}{||y||_1}-\frac{x}{||x||_1}}_2}{\norm{\frac{y}{||y||_1}}_2},$$
so that the criterion of \cite{negahban2016rank} can be written simply as $D(\widehat{W}, w)$. 

We will show that if $\widehat{W}$ and $w$ satisfy $\max_{i,j} w_i/w_j \leq b$ and $\max_{i,j} \widehat{W}_i/\widehat{W}_j \leq b$, then the two relative error criteria are within a multiplicative factor of $\sqrt{b}$. Thus, ignoring factors depending on the the skewness $b$, we may pass from one to the other at will. 

The proof will require a sequence of lemmas, which we present next. The first lemma provides some inequalities satisfied by the the sine error measure. 

\begin{lemma_ap}\label{lem:sinus}
Let $x,y\in \Re^n$ and denote by $\sin(x,y)$ the sine of the angle made by these vectors. Then we have that
$$
|\sin(x,y)| = \min_\beta \frac{\norm{\beta x- y}_2}{\norm{y}_2}  = \inf_{\alpha\neq 0} \frac{\norm{x-\alpha y}_2}{\norm{\alpha y}_2} 
$$
Moreover, if the angle between $x$ and $y$ is less than $\pi/2$ (which always holds when $x$ and $y$ are nonnegative), we have that \begin{small}
\begin{equation}\label{eq:sin_norm}
\frac{1}{\sqrt{2}}\norm{\frac{x}{\norm{x}_2}-\frac{y}{\norm{y}_2}}_2 
\leq |\sin(x,y)|\leq \norm{\frac{x}{\norm{x}_2}-\frac{y}{\norm{y}_2}}_2. 
\end{equation} \end{small}
\end{lemma_ap}
Moreover, since $\sin(x,y)=\sin(y,x)$ the expressions remain valid if we permute $x$ and $y$.
\begin{proof}
We begin with the first equality. Observe that
$\min_{\beta} \norm{\beta x- y}_2$ is the distance between $y$ and its orthogonal projection on the 1-dimensional subspace spanned by $x$; by definition of  sine, this is also 
$\norm{y}_2 \abs{\sin(x,y)}$, which implies the equality sought.

The second equality directly follows from the change of variable $\alpha = 1/\beta$.  Passing from $\min$ to $\inf$ is necessary  is necessary in case the optimal $\beta$ is 0, which happens when $x$ and $y$ are orthogonal.

Let now $\theta$ be the angle made by $x$ and $y$. An analysis of the triangle defined by 0, 
$x/\norm{x}_2$ and  $y/\norm{y}_2$ shows that $\sin(x,y) = \sin(\frac{\pi-\theta}{2}) \norm{\frac{x}{\norm{x}_2}-\frac{y}{\norm{y}_2}}_2 $, which implies (\ref{eq:sin_norm}) since $\theta\in [0,\frac{\pi}{2}]$.
\end{proof}

We will also need the following lemma on the ratio between the $1$- and $2$- norms of vectors. 
\begin{lemma_ap}\label{lem:ratio_norm}
Let $x\in \Re^n_+$ be such that $\max_{i,j} \frac{x_i}{x_j} \leq b$. Then 
$$
\frac{\norm{x}_{2}}{\norm{x}_1}\leq \min\prt{1,\sqrt{\frac{b}{n}}}.
$$
\end{lemma_ap}
\begin{proof} That $||x||_2 \leq ||x||_1 \cdot 1$ is well-known. To prove the same with $1$ replaced by $\sqrt{\frac{b}{n}}$, we argue as follows. \red{First, without loss of generality, we may assume $x_i \in [1,b]$ for all $i$. Let $Z$ be any random variable supported on the interval $[1,b]$. Observe that \[ E[Z^2] \leq b E[Z] \leq b E[Z]^2,\] where the first inequality follows because $Z \leq b$ and the second inequality follows because $E[Z] \geq 1$. We can rearrange this as 
\[ \frac{E[Z]^2}{E[Z^2]} \geq \frac{1}{b}. \] 
Now let $Z$ be uniform over $x_1, \ldots, x_n$. In this case, this last inequality specializes to 
\[ \frac{\left((1/n) \sum_{i=1}^n x_i \right)^2}{(1/n) \sum_{i=1}^n x_i^2} \geq \frac{1}{b},\] or \[ \frac{||x||_1^2}{||x||_2^2} \geq \frac{n}{b}, \] and now, inverting both sides and taking square roots, we obtain what we need to show.}
\end{proof}

\begin{lemma_ap}
\red{$\abs{\sin(x,y)} \leq  D(x,y)$} holds for nonnegative $x,y\in \Re^n$.
\end{lemma_ap}
\begin{proof}
\red{\begin{eqnarray*} D(x,y) & = &  \frac{\left| \left|\frac{x}{||x||_1} - \frac{y}{||y||_1} \right| \right|_2}{\left| \left|\frac{y}{||y||_1} \right| \right|_2} \\ & = &  \frac{\left| \left|x \frac{||y||_1}{||x||_1} - y \right| \right|_2}{||y||_2} \\ & \geq & \inf_{\beta} \frac{||\beta x - y||_2}{||y||_2} \\
& = & |\sin(x,y)|, 
\end{eqnarray*} where the last step used Lemma \ref{lem:sinus}.}
\end{proof}

\begin{lemma_ap}\label{lem:compare_criteria}
Suppose $x \in \Re^n_+$ and $\max _{i,j} \frac{x_i}{x_j} \leq b$. Then there holds
$$
D(x,y) \leq \min\prt{1+\sqrt{n} , 1+\sqrt{b}}\sqrt{2} \sin(x,y)
$$
\end{lemma_ap}
\begin{proof}
Without loss of generality, we assume $\norm{x}=\norm{y} = 1$, which means we can simplify $\norm{\frac{x}{\norm{x}_2}-\frac{y}{\norm{y}_2}}_2$ as $\norm{x-y}_2$. Since $\norm{y}_3 = 1$, we have
\begin{eqnarray*}
D(x,y) &=& \frac{\norm{\frac{x}{1^Tx}-\frac{y}{1^Ty} }_2}{\norm{\frac{y}{1^Ty}}_2}\nonumber\\
&\leq& \norm{\frac{1^Ty}{1^Tx}x - y }_2\nonumber\\
&\leq& \norm{x-y}_2 + \norm{x}_2 \abs{\frac{1^Ty}{1^Tx}-1}\nonumber\\
& =& \norm{x-y}_2 + \frac{\norm{x}_2}{1^Tx}  \abs{1^T(y-x)}\nonumber\\
& \leq& \norm{x-y}_2 \prt{1+ \sqrt{n}\frac{\norm{x}_2}{\norm{x}_1} }, \label{eq:temp_bound_dxy}
\end{eqnarray*}
where in the last inequality we have used $$\abs{1^T(y-x)}\leq \norm{y-x}_1 \leq \sqrt{n} \norm{y-x},$$ and 
$||y-x||_2 \leq 1$ due to the positivity of $x$ and $y$. 
Now using Lemma \ref{lem:sinus} to bound $||x-y||_2 \leq \sqrt{2} \sin(x,y)$, we have that the first part of the bound follows then from $\norm{x}_2 \leq \norm{x}_1$, and the second one from Lemma \ref{lem:ratio_norm}.
\end{proof}

\section{Supplementary Information: proof of Theorem \ref{thm:lower_resistance}} \label{sec:proof_lower_res}

Our starting point is a lemma from \cite{slt}, which we will use throughout the lower bound proofs, and which we introduce next. 

Let $d(w,w')$ be a metric on ${\cal W}\times {\cal W}$. Let $P_w(y)$ be an indexed family of probability distributions on the observation space ${\cal Y}$. Let $\widehat{w}(y)$ be an estimator based on observations $y \in {\cal Y}$ and let $\bY$ represent the random vector associated with the observations conditioned on $w$. We use $E_{\bY}[\cdot]$ to denote expectation with respect to the randomness in $\bY$. 

We first lower bound the worst-case error by means of a Bayesian prior. Namely, we observe that if we generate $w$ according to some distribution $\pi$, then using $E_{\pi}[\cdot]$ to denote expectation when $w$ is generated this way, we have
\begin{equation}\label{eq:bayes_prior_bound}
\sup_{w\in {\cal W}}\mathbb{E}_{\bY}[d(w,\widehat{w}(\bY)] \geq \mathbb{E}_{\pi,\bY}[d(w,\widehat{w}(\bY)]]\end{equation}

We will use [\cite{slt} Chap. 13, Corollary 13.2] to obtain a lower bound on \red{(components of)} the latter quantity. 

\begin{lemma_ap} \label{lem:bayeslb}
Let $\pi$ be any prior distribution on ${\cal W}$, and let $\mu$ be any joint probability distribution of a random pair $(w,w') \in {\cal W}\times {\cal W}$, such that the marginal distributions of both $w$ and $w'$ are equal to $\pi$. Then 
\[
\mathbb{E}_{\pi,\bY}[d(w,\widehat{w}(\bY)]]\geq \mathbb{E}_{\mu} [d(w,w')(1-\|P_w - P_{w'}\|_{\mbox{TV}}]    
\]
\noindent where $||\cdot||_{\rm TV}$ represents the total-variation distance between distributions.
\end{lemma_ap}

We will need a slight generalization of the Lemma for our purposes. In particular, we note that it is sufficient that the measure $d(w,w')$ satisfies a weak version of triangle inequality, i.e., $\gamma d(w_1,w_2)\leq d(w_1,\widehat{w})+d(w_2,\widehat{w})$ for some pre-specified constant $\gamma$. Following along the same lines as the proof of Le-Cam's two-point method in [\cite{slt}] we get: \begin{small}
\begin{eqnarray} \label{eq.relaxtriangeq}
\sup_{w\in {\cal W}}\mathbb{E}_w[d(w,\widehat{w})] \geq \gamma \mathbb{E}_{\mu} [d(w,w')(1-\|P_w - P_{w'}\|_{\mbox{TV}}]    
\end{eqnarray} \end{small}

\noindent 
Next, to apply this lemma we need to associate the random variables of interest in our problem with the the measure $P_w$. The random variable $Y_e$ and the corresponding observations $y_e$ are associated with the edge $e \in E$ of our graph. In particular, let $B_e$ be the e$^{th}$ row of $B$. Recall that $BB^T$ is the graph Laplacian. For an edge $e=(ij)$, let $y_e=1$ if $i$ wins over $j$ and $-1$ otherwise.

We now define our distribution $\pi$:
Let $B=\sum_{i=1}^n \sigma_i u_i v_i^T$ be a singular decomposition of $B$. We augment the collection of singular vectors $\sigma_i, v_i, i=1,2,\ldots,d$ with the constant vector $v_0=\frac{1}{\sqrt{n}} {\bf 1}$. We observe that this collection $V=[v_0,v_1 \ldots, v_n]$ forms an orthonormal basis. We overload notation and collect the observations, $y_e,\,e \in E$ into a vector $\mathbf{y}$ and the corresponding random-variable $\bY$. 
We specify define $\pi(w)$ by placing a uniform distribution on the hypercube $\{-1,1\}^{n}$. We then let 
$z=(z_1,\ldots, z_n) \sim \mbox{Unif}\{-1,1\}^n$ and write: 
\begin{equation}
w_z = V\Lambda z = \sqrt{n} v_0 + \delta \sum_{i=1}^n\frac{z_i}{\sigma_i}v_i    
\end{equation}
where, $\delta$ is a suitably small number to be specified later. So, in particular, $\lambda_0 = \sqrt{n}$ and $\lambda_i=\delta/\sigma_i$ for $i=1,2,\ldots, n$. We note that the norm of $w_z$'s defined this way are all equal, i.e., 
\begin{eqnarray}
\|w_z\| &=& \|V\Lambda z \| =  \|\Lambda z \| \nonumber\\ 
 &=& \sqrt{n + \delta^2 \sum_{i=1}^n \frac{1}{\sigma_i^2}}
\label{eq:norm_Lambda_z}
\end{eqnarray}

\red{Our (square) error criterion $\sin^2(\hw,w)$, is lower bounded by 
$$
\frac{1}{2}\rho(w,\widehat{w}):= \frac{1}{2} \left \|\frac{w}{\|w\|}-\frac{\widehat{w}}{\|\widehat{w}\|} \right \|^2 
= \rho(w,\widehat{w}),
$$
see Lemma \ref{lem:sinus}.}

Next, we closely follow the argument in the proof of Assouad's lemma~[\cite{slt}]. To do this we need to express $\rho(w,\widehat{w})$ as a decomposable metric. To this end, let $\widehat{\alpha}(y) = V^{T} \widehat{w}(y)$.
We will suppress dependence on $y$ when it is clear from the context. 
We write:
$$
\min\limits_{\widehat{w}(\bY)} \mathbb{E}_{\pi,\bY}[\rho(w,\widehat{w}(\bY))] = \min\limits_{\widehat{w}} \mathbb{E}_{\pi,\bY} \left \| \frac{w}{\|w\|} - \frac{\widehat{w}}{\|\widehat{w}\|} \right \|^2 
$$\begin{eqnarray}
&=&  \min\limits_{\widehat{w}(\bY)} \mathbb{E}_{\pi,\bY} \left \| V^{T}\left (\frac{w}{\|w\|} - \frac{\widehat{w}}{\|\widehat{w}\|}\right ) \right \|^2 
\nonumber\\
&=& \min\limits_{\widehat{\alpha}(\bY)} \mathbb{E}_{\pi,\bY} \sum_{i=0}^n \left (\frac{\lambda_i z_i}{\|\Lambda z\|} - \frac{\widehat{\alpha_i}}{\|\widehat{\alpha}\|}\right )^2 \nonumber\\
&\geq& \sum_{i=1}^n \min\limits_{\beta_i(\bY)} \mathbb{E}_{\pi,\bY} \left (\frac{\lambda_i z_i}{\|\Lambda z\|} - \beta_i(\bY)\right )^2\nonumber\\
&=&\sum_{i=1}^n \min\limits_{\eta_i(\bY)}\frac{\lambda_i^2}{\|\Lambda z\|^2} \mathbb{E}_{\pi,\bY} \left(z_i - \eta_i(\bY) \right)^2,
\label{eq:decomposition_rho}
\end{eqnarray}
where $\beta_i(\bY),\eta_i(\bY)$ are estimators using the whole vector $\bY$ for each $i$, and the last equality follows from $\|\Lambda z\|$ being constant over the support of $z$.
We are now going to apply the variation (\ref{eq.relaxtriangeq}) of Lemma \ref{lem:bayeslb} to each  $\mathbb{E}_{\pi,\bY} d_i(z,\eta_i(\bY)):=   \mathbb{E}_{\pi,\bY} \left(z_i - \eta_i(\bY) \right)^2$  individually. For this purpose, we define the distribution $\mu_i(z,z')$ by keeping $z$ uniformly distributed in $\{-1,1\}^n$, and flipping the $i^{th}$ bit to obtain $z'$ (formally, $z'_i = -z_i$ and $z'_j = z_j$ for every $j\neq i$). Clearly, $\mathbb{E}_{\pi,\bY} d_i(z,z') = 4$. 
We next work on simplifying the total variation (TV) term in the expression of Lemma~\ref{lem:bayeslb}. First, note that since we have $k$ independent observations per-edge, we tensorize the probability distributions and denote it as $P_w^{\otimes k}$. By the Pinsker's lemma it follows that the total variation distance can be upper-bounded by the the Kullback-Leibler Divergence [\cite{slt}], and furthermore, it follows from standard algebraic manipulations (see [\cite{duchi} Example 3.4]) that,
\begin{eqnarray} \label{eq:logitbound}
\|P_w^{\otimes k} - P_{w'}^{\otimes k}\|_{\mbox{TV}}^2 & \leq& \frac{1}{2} D_{KL}(P_w^{\otimes k}\|P_w'^{\otimes k})\\ \nonumber &\leq& \frac{k}{4}\|B(\log(w)-\log(w'))\|^2.
\end{eqnarray}  Indeed, recall that  the probability of $i$ winning over $j$ is $\frac{w_i}{w_i+w_j} = \frac{1}{1+w_j/w_i}$, and observe that $B_e\log(w) = \log(w_i/w_j)$. Hence we can write
\begin{small}$$P_w(y_e) \triangleq \mbox{Prob}[Y_e=y_e \mid B_e,w] = \frac{1}{1+\exp(-y_eB_e\log(w))}.$$\end{small}Thus $P_w$ and $P_{w'}$ satisfy the ``logistic regression'' distribution, and [\cite{duchi} Example 3.4]) derives Eq. (\ref{eq:logitbound}) for total variation distance between such distributions. 

Now we prove in Section \ref{sec:proof_upperKL} below that for $\delta \sigma_{\max}n\res_{avg}\leq 1$ and \red{$\delta^2 n \Omega_{\rm avg}/2 \leq 1/4$}, we have,
\begin{equation} \label{eq:upperKLlogit}
\|B(\log(w)-\log(w'))\|^2 \leq 16\delta^2.     
\end{equation}
Hence it follows from (\ref{eq.relaxtriangeq}) that for every estimator $\eta_i(\bY)$ and for such $\delta$, 
$$
\mathbb{E}_{\pi,\bY} \left(z_i - \eta_i(\bY)\right)^2 \geq 
\gamma 4(1-\sqrt{4k\delta^2}), 
$$
and then from (\ref{eq:decomposition_rho}) that
\begin{eqnarray*}
\min\limits_{\widehat{w}(\bY)} \mathbb{E}_{\pi,\bY}[\rho(w,\widehat{w}(\bY))] &\geq& \gamma\sum_{i=1}^n
\frac{\lambda_i^2}{\|\Lambda z\|^2}
4(1-\sqrt{4k\delta^2})\\
\\&\geq& \gamma
\sum_{i=1}^n
\frac{4\delta^2(1-\sqrt{4k\delta^2})}{\sigma_i^2n}\\
&=& 2\gamma\delta^2(1-\sqrt{4k\delta^2}) \frac{n-1}{n}\res_{avg},
\end{eqnarray*}
where we have used $\sum_i \frac{1}{\sigma_i^2} = {\rm tr}(L^\dag) = \frac{n-1}{2} \res_{\rm avg}$. The result of Theorem \ref{thm:lower_resistance} follows then from taking $\delta^2 = \frac{1}{16k}$. We need to make sure that the conditions $\delta \sigma_{\max}n\res_{avg}\leq 1$ and \red{$\delta^2 n \Omega_{\rm avg}/2 \leq 1/4$} are satisfied, and for that it suffices to take $k \geq c \sigma_{\rm max} n \res_{\rm avg}$ for some absolute constant $c$. Finally, recall that $\sigma_{\rm max}$ is the largest singularvalue of $B$, and $L=B B^T$, so that $\sigma_{\rm max} = \sqrt{\lambda_{\rm max}(L)}$, so the condition we need can be written as $k \geq c \sqrt{\lambda_{\rm max}(L)} n \res_{\rm avg}$.

\subsection{Proof of Equation (\ref{eq:upperKLlogit}) }\label{sec:proof_upperKL}

In this subsection, we complete the proof by providing a proof of Eq. (\ref{eq:upperKLlogit}). Our starting point is the observation  that, $\log([w_z]_\ell)=\log(1+ \delta \sum_{j=1}^n v_{\ell j}\frac{z_j}{\sigma_j})$. Noting that by Cauchy-Schwartz inequality 
\begin{eqnarray}
\abs{\delta \sum_{j=1}^n v_{\ell j}\frac{z_j}{\sigma_j}} &\leq& 
\sqrt{\delta^2 \prt{\sum_{j=1}^n \frac{1}{\sigma_j^2}}}\nonumber\\
&=& \sqrt{\delta^2 \frac{n-1}{2}\res_{avg}}\nonumber\\
&\leq& \sqrt{\delta^2 n\res_{avg}/2}\label{bound_hz}
\end{eqnarray}
we enforce the constraint that $\delta$ should be sufficiently small so \begin{equation}\label{eq:bound_delta_tech}\delta^2 n \res_{avg}/2 \leq 1/4.\end{equation} 
This constraint enables us to use a Taylor approximation for $\log([w_z]_\ell)-\log([w_{z'}]_\ell)$. 

\red{We use the Taylor's expansion 
\[ f(x) = f(1) + f'(1) (x-1) + \frac{1}{2} f''(\xi) (x-1)^2, \] for the function
$f(x) = \log (x)$. This gives us 
\[ \log x = x - 1 + \frac{1}{2} f''(\xi) (x-1)^2,\] where $\xi$ belongs to the interval between $1$ and $x$. In particular, 
\begin{eqnarray*} \log ([w_z]_l) & = & \log (1 + \delta \sum_{j} \frac{z_j}{\sigma_j} [v_j]_l) \\ 
& = & \delta \sum_{j} \frac{z_j}{\sigma_j} [v_j]_l + C_l \delta^2 (\sum_j \frac{z_j}{\sigma_j} [v_j]_l)^2 ,
\end{eqnarray*} where because of Eq. (\ref{bound_hz}) and our bound on $\delta$, we have that $C_l$ is upper bounded by $(1/2)f''(1/2) = 2$.

Similarly, 
\[ \log ([w_{z'}]_l =  \delta \sum_{j} \frac{z_j'}{\sigma_j} [v_j]_l + C_l \delta^2 (\sum_j \frac{z_j'}{\sigma_j} [v_j]_l)^2,\] where $C_{l'}$ is lower bounded by $(1/2)f''(3/2) = 2/9$.

Observe that, according to our joint distribution over the pair $(w,w')$, we have the bit $i$ flipped, while all others remain the same, namely, $z_j=z'_j$ for $j\neq i$ and $z_i=-z'_i$. Thus \begin{small}
\[  \log ([w_z]_l) -  \log ([w_{z'}]_l =  2 \delta \frac{z_i}{\sigma_i} [v_i]_l+ (C_l - C_{l'}) \delta^2 (\sum_j \frac{z_j'}{\sigma_j} [v_j]_l)^2 \] \end{small}We can write this as 
\[ \log w_z - \log w_{z'} = 2 \delta \frac{z_i}{\sigma_i} v_i + \delta^2 h_z. \]
Recalling that $V$ is the vector that stacks up the vectors $v_i$ as columns, we then have
\begin{eqnarray*} ||h_z||_2 & \leq &  ||h_z||_1 \\ 
& = &  \sum_l (2-2/9) (\sum_{j \neq i,0} \frac{z_j}{\sigma_j} [v_j]_l)^2 \\ 
& \leq &  \sum_l 2 (\sum_{j \neq i,0} \frac{z_j}{\sigma_j} V_{lj})^2 \\ 
& = &   2 (\sum_{j \neq i,0} [V ({\rm diag}(\sigma)^{-1} z )]_j)^2 \\
& \leq & 2 ||{\rm diag}(\sigma)^{-1} z||_2^2 \\ 
& = & 2 \sum_{j=1}^n \frac{1}{\sigma_j^2} \\ 
& = & 2 {\rm tr}(L^\dag) \\ 
& \leq & 2 n \Omega_{\rm avg}. 
\end{eqnarray*} }
This leads us to:
\begin{eqnarray*}
\|B(\log(w_z)-\log(w_{z'}))\|\hspace{-.3cm}&\leq&\hspace{-.3cm} \frac{2\delta}{\sigma_i}\|B v_i\| + \delta^2\|B (h_z-h_{z'})\| \\ &\leq&\hspace{-.3cm} 2\delta + 4\delta^2 \sigma_{\max}n\res_{avg}. 
\end{eqnarray*}

Under the assumption that that $\delta$ is small enough so that 
\[ \delta \sigma_{\rm max} n \Omega_{\rm avg} \leq 1 \] we obtain that 
\[ \|B(\log(w_z)-\log(w_{z'}))\| \leq 4 \delta, \]
which is what we needed to show. 

\section{Supplementary Information: proof of Lemma \ref{highprob}\label{highprobproof}}

We use the following version of Chernoff's inequality: if $Y_l$ are  are independent random variables with zero expectation, variances $\sigma_l^2$, and further satisfying $|Y_l| \leq 1$ almost surely, then \begin{equation} \label{eq:chernoff} P \left( \left| \sum_{l=1}^K Y_l \right|  \geq \lambda \sigma \right) \leq C \max \left( e^{-c \lambda^2}, e^{-c \lambda \sigma} \right), \end{equation}  for some absolute constants $C,c>0$,  where $\sigma^2 = \sum_{i=1}^k \sigma_i^2$ (see Theorem 2.1.3 of \cite{tao2012topics}). Note that when $\lambda \leq \sigma$, this reduces to 
\begin{equation} \label{eq:chernoff2} P \left( \left| \sum_{l=1}^K Y_l \right|  \geq \lambda \sigma \right) \leq C  e^{-c \lambda^2}. \end{equation}

Let $X_{ij}^l$ be the outcome of the $l$'th coin toss comparing nodes $i$ and $j$; that is, $X_{ij}^l$ is an indicator variable equal to one if $i$ wins the toss. We let $Y_l = X_{ij}^l - p_{ij}$.  Then  $Y_l$ are independent random variables, $|Y_l| \leq 1$, and thus we can apply Eq. (\ref{eq:chernoff}). Note that $\sigma_l^2 = 1/v_{ij}$ as shown in (\ref{eq:variance_Xij}). 

We apply Eq. (\ref{eq:chernoff}) with the choice of $\lambda = \sqrt{C_{n,\delta}}$. Choosing $k \geq 4 b C_{n,\delta}$, i.e. $c_2\geq 4$ in view of Assumption \ref{assump:ratio_constant}, and using that $v_{ij} \leq 4b$, it follows that \[ \lambda^2 = C_{n,\delta} \leq \frac{k}{v_{ij}} = \sigma^2, \] so that $\lambda \leq \sigma$. Thus Eq. (\ref{eq:chernoff}) reduced to Eq. (\ref{eq:chernoff2}), which yields 
\[ P \left( \left| kF_{ij} - k p_{ij}  \right| \geq \sqrt{C_{n,\delta}} \sqrt{k/v_{ij}} \right) \leq C e^{ -c C_{n, \delta}} \leq \frac{\delta}{n^2}, \] where this last inequality requires a suitable choice of the constant $c_1$, and we remind that $k F_{ij}$ is the number of successes of $i$ over $j$, and. Applying the union bound over the $|E| \leq n^2$ pairs $i,j$ yields the result.

\section{Supplementary Information on the experiments in Section \ref{sec:expe}}\label{sec:info_expe}

We first note that we implemented a minor modification of our algorithm: Our estimators (\ref{eq:defLS}) use $\log R_{ij}$, and are thus not defined when the ratio $R_{ij}$ of wins is zero or infinite, i.e. when one agent wins no comparison with one of its neighbors. To avoid this problem, we artificially assign half a win to such agents. Note that these events are typically rare, and their joint probability tends to zero when $k$ grows. Our error analysis can actually be shown to remain valid for our modified algorithm.

Each data point in the curves presented in Section \ref{sec:expe} corresponds to the average error $|\sin(\hw,w)|$ on a number $\Ntest$ of independent trials, chosen sufficiently large so that the curves are stables. The weights $w_i$ were independently randomly generated for each node $i$, with $\log w_i$ following a uniform distribution between 0 and $\log b$. For experiments on Erdos-Renyi graphs, a new graph was created at each trial. Disconnected graphs were discarded, so the results should be understood as conditional to the graph being connected.


\end{document}